\documentclass{article}
\usepackage{arxiv}
\usepackage{amssymb}
\usepackage{kbordermatrix}
\usepackage[utf8]{inputenc} 
\usepackage[T1]{fontenc}    
\usepackage{hyperref}       
\usepackage{url}            
\usepackage{booktabs}       
\usepackage{amsfonts}       
\usepackage{nicefrac}       
\usepackage{microtype}      
\usepackage{lipsum}		
\usepackage{graphicx}
\usepackage{natbib}
\usepackage{doi}
\usepackage{amsmath}
\usepackage{physics}
\usepackage{wrapfig}
\usepackage{bm}
\usepackage[linesnumbered,ruled]{algorithm2e}
\usepackage{xcolor}
\usepackage{forest}
\usepackage{tikz-cd}
\usepackage{enumitem}

\newenvironment{proof}{\paragraph{Proof:}}{\hfill$\square$}

\newtheorem{theorem}{Theorem}

\newtheorem{lemma}[theorem]{Lemma}

\date{} 					

\title{Optimal Cooperative Multiplayer Learning Bandits with Noisy Rewards and No Communication}
\author {
    William Chang\textsuperscript{\rm 1},
    Yuanhao Lu\textsuperscript{\rm 2},\\
    \textsuperscript{\rm 1}University of California, Los Angeles\footnote{chang314@g.ucla.edu}\\
    \textsuperscript{\rm 2}Princeton University\footnote{terrylu@princeton.edu}\\
}

\begin{document}
\maketitle

\begin{abstract}

We consider a cooperative multiplayer bandit learning problem where the players are only allowed to agree on a strategy beforehand, but cannot communicate during the learning process. In this problem, each player simultaneously selects an action. Based on the actions selected by all players, the team of players receives a reward. The actions of all the players are commonly observed. However, each player receives a noisy version of the reward which cannot be shared with other players. Since players receive potentially different rewards, there is an asymmetry in the information used to select their actions. In this paper, we provide an algorithm based on upper and lower confidence bounds that the players can use to select their optimal actions despite the asymmetry in the reward information. We show that this algorithm can achieve logarithmic $O(\frac{\log T}{\Delta_{\bm{a}}})$ (gap-dependent) regret as well as $O(\sqrt{T\log T})$ (gap-independent) regret. This is asymptotically optimal in $T$. We also show that it performs empirically better than the current state of the art algorithm for this environment.
\end{abstract}

\section{Introduction}
The field of stochastic Multi-armed Bandit (MAB) contains some of the most well-studied problems in reinforcement learning. These online learning algorithms are designed to understand how to implement exploration-exploitation tradeoffs to achieve the most reward. The classical version of MAB consists of a single agent with a set of $[m]:= \{1,\ldots,m\}$ actions to choose from. Each action is associated with an unknown reward distribution that is sub-Gaussian. A \emph{round} in a MAB environment is defined as one iteration where the player selects an action and obtains a reward. At every such round, the agent's goal is to select the arm with the highest expected reward. \footnote{We term `arm' and `action' interchangeably} One can measure the success of a policy for arm selection using the notion of \emph{regret}, which measures how often a suboptimal arm is chosen (pulled). The goal is to minimize the regret for large horizons. In the classical single-agent setting, \cite{lai1985asymptotically} showed that every policy will not be able to perform better than $O(\log T)$ gap-dependent regret. This lower bound on regret was first attained by the UCB algorithm. 


Single-player MABs, however, inadequately model the complexities of real-world applications involving multiple interacting entities. This gap has sparked a recent growing interest in cooperative multiplayer MAB problems, where multiple agents are maximizing their collective expected rewards. For example, \cite{gai2012combinatorial, liu2010distributed, anandkumar2011distributed, kalathil2014decentralized, nayyar2016regret} introduces and applies multiplayer MAB models for spectrum sharing in wireless networks. These papers assume each player's reward is independent of the other's actions (i.e. joint actions are not considered). 

While the aforementioned works successfully extend the classical MAB problem to include multiple players, they remain restrictive. Specifically, they fail to adequately model scenarios in which the decisions of one agent might influence the rewards of others, a common interaction arising in real-world settings. For instance, in shared network settings, the bandwidth an agent consumes directly affects the network's capacity for other users. Similarly, in financial markets, buying or selling decisions made by one trader can influence the stock price and, consequently, the rewards for other traders. To provide a framework to model these problems, we consider multiagent settings where each agent has their own (marginal) set of actions to choose from. Subsequently, all agents' collective actions form a joint action across all players. Note that this is different than leader-follower games \cite{yu2022learning}, where the leader selects an action first, which the followers observe before they select their actions. 

Furthermore, our paper considers the multiagent setting where \emph{information asymmetry} is preset. By information asymmetry, we mean there exists some information that is not shared among all players. Specifically, we analyze the abovementioned multiplayer setting where \textit{reward asymmetry} is present; that is, agents do not observe the rewards obtained by other agents. Information asymmetry naturally occurs when communication between agents is restricted. Thus, effectively, we do not allow for any form of communication between agents, although they are allowed to agree on a policy before the learning phase begins and know the number of actions the other players have. This setting is rich in real-world motivation, where agents collaboratively pursue a common objective despite the absence of direct inter-agent communications. For example, in decentralized traffic control systems where individual agents — traffic lights in disparate regions — endeavor to optimize the average vehicular travel time. In this application, the actions of individual agents depend on each other, and local traffic conditions are not observed by traffic lights in remote areas, necessitating the use of joint actions and reward asymmetry as investigated in this study. 

\textbf{Our Contributions}
We propose merry go around variant of the UCB algorithm we call \texttt{mUCB-Intervals}. The novelty of this algorithm is combining an interval method for best arm selection \cite{audibert2010best}, applying it to regret minimization tasks, and including an aspect of coordination for the multiplayer setting. More explicitly, the players will decide on an ordering of the arms prior to the learning (they can use the ordering given in \cite{chang2021online}), and during the process, they will pull each arm in order. They will also maintain a \emph{desired set} in which all the arms will be in this set at the beginning. This desired set will remain the same for all players for each round. Based on the UCB "error" intervals each player will determine if the next arm should be in the set, and if not, they will communicate this to the other players by not pulling what should have been the next arm. Note that there is no explicit communication in the environment but the players know when they eliminate the next arm from the desired set, and thus can maintain the same desired set. This is similar to other bandit works in collision sensing \cite{boursier2019sic} but their communication scheme is much simpler in that they simply have to observe when other bandits have pulled the same arm. 

We show that our algorithm achieves $O(\frac{\log T}{\Delta_{\bm{a}}})$ gap dependent regret or $O(\sqrt{T\log T})$ gap independent bound. Whereas \cite{chang2021online} was able to achieve \emph{almost} optimal regret for this setting, \texttt{mUCB-Intervals} is the first algorithm to achieve optimal regret for a reward asymmetric setting. It's easy to implement and understand and performs better empirically against \texttt{mDSEE} from \cite{chang2021online}, the current state of the art algorithm for this environment.

\paragraph*{Related Work.}
The literature on multi-armed bandits is overviewed in \cite{sutton2018reinforcement,gittins2011multi,lattimore2020bandit}. Some classical papers worth mentioning are \cite{lai1985asymptotically,anantharam1987bandits,auer2002finite}. Interest in multi-player MAB models was triggered by the problem of opportunistic spectrum sharing and some early papers were \cite{gai2012combinatorial,liu2010distributed,anandkumar2011distributed}. Other papers motivated by similar problems in communications and networks are \cite{maghsudi2014channel,korda2016distributed,shahrampour2017multi,chakraborty2017coordinated}. These papers were either for the centralized case, or considered the symmetric user case, i.e., all users have the same reward distributions. Moreover, if two or more users choose the same arm, there is a ``collision", and neither of them get a positive reward.  The first paper to solve this matching problem in a general setting was \cite{kalathil2014decentralized} which obtained log-squared regret. It was then improved to log regret in \cite{nayyar2016regret} by employing a posterior sampling approach. These algorithms required implicit (and costly) communication between the players. Thus, there were attempts to design algorithms without it \cite{avner2014concurrent, rosenski2016multi, bistritz2018distributed, feraud2019decentralized, boursier2019sic}. Other recent papers on decentralized learning for multiplayer matching MAB models are \cite{wang2020optimal, shi2020decentralized}.

In the realm of multiplayer stochastic bandits, many works allow for limited communication such as those in \cite{martinez2018decentralized, martinez2019decentralized, szorenyi2013gossip, karpov2020collaborative, tao2019collaborative}. An exception is \cite{bistritz2021one} where all the players select from the same set of arms and their goal is to avoid a collision, that is, they do not want to select the same arm as another player. Another work that doesn't allow for communication \cite{xu2015distributed} where they developed  online learning algorithms that enable agents to cooperatively learn how to maximize reward with noisy global feedback without exchanging information. Recently, in light of the work from \cite{chang2021online}, there have been other works which have studied information asymmetric multiplayer bandits \cite{mao2022improving, kao2022decentralized, mao2021decentralized, kao2022efficient}.

\section{Preliminary}

\begin{figure*}
\includegraphics[width = \textwidth]{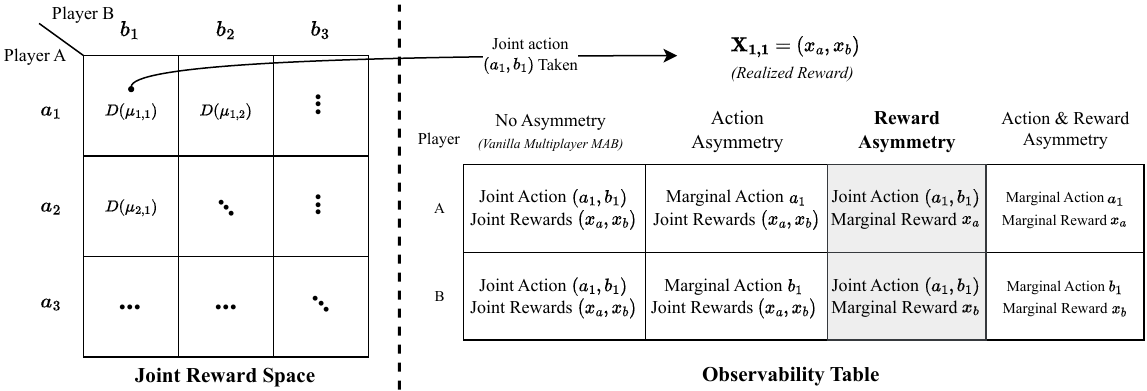}
\caption{A two-player information asymmetric in reward setting. The grid on the left is a visualization of the joint action space, where the rows correspond to the actions available to Player A while the columns correspond to the actions available to player B. Each entry in this grid has a subgaussian reward distribution $D(\mu)$ for some mean $\mu$. The table on the right lists what is observed by each player for the different types of information asymmetry. In our setting in gray, for each player, the joint actions are observed but only their copy of the IID reward is observed.}
\label{fig:problem_setting}
\end{figure*}

We follow the formulation of \textit{information asymmetry} only in rewards given in \cite{chang2021online} which we reiterate here for completeness: 

Consider a set of $M$ players $P_{\bm{1}},\cdots,P_M$, in which player $P_i$ has a set ${\mathcal K}_i$ of $K_i$ arms to pick from. At each time instant, each player picks an arm independently and simultaneously from other players from their set $\mathcal{K}_i$. The joint action can be interpreted as an $M$-tuple of arms picked denoted by $\bm{a}=(a_{\bm{1}},\cdots,a_M)$. We shall use bold font to denote vectors. For simplicity, we shall assume each player has $K$ actions to pick from which gives a total of $K^M$ joint actions.  This generates $M$ independent and identically distributed (iid) random reward $X_{\bm{a}}^i \in [0,1]$ with $i \in [M]$ from a $1$-subgaussion distribution $F_{\bm{a}}$  with mean $\mu_{\bm{a}}$. Each player $P_i$ can only observe their rewards $X_{\bm{a}}^i$ and are oblivious to all the other rewards (this is the reward asymmetry). However, they can observe the actions of all other players at all times after the joint action has been taken. Note that this is different from leader-follower games in that all players select their actions at the same time. All the players know the joint action space beforehand and they can decide on a strategy a priori. However, during learning they are not allowed to communicate in any way.

 Denote $\Delta_{\bm{a}} = \mu^* - \mu_{\bm{a}}$ where $\mu^*$ is the highest reward mean among all arm tuples, and we shall call the corresponding arm the \emph{optimal arm}. The players share a goal and that is to pull the optimal arm as often as possible. Let $a_t[i] \in \{1,\cdots,K_i\}$ be the arm chosen by player $i$ at time $t$, and denote $\bm{a_t} = (a_t[1],\cdots,a_t[m])$. A high-level objective is for the players to collectively identify the best set of arms $a^*$ corresponding to mean reward $\mu^*$. But the players do not know the means $\mu_{\bm{a}}$, nor the distributions $F_{\bm{a}}$. They must learn by playing and exploring. Thus, we can capture learning efficiency of an algorithm via the notion of \textit{expected regret},
\begin{equation}\label{eq:regret}
	R_T = \mathbb{E}\left[T\mu^* - \sum_{t=1}^T X_{\bm{a(t)}}(t)\right]
\end{equation}
where $T$ is the number of learning instances  and $X_{\bm{a(t)}}(t)$ is the random reward if arm-tuples $a(t)$ are pulled. We let $n_{\bm{a}}(t)$ be the number of times joint arm $\bm{a}$ has been pulled up  to(not including) time $t$. Furthermore let $\widehat{\mu}_{\bm{a}}^i(n_{\bm{a}}(t))$ be the empirical mean of arm $\bm{a}$ for player $i$ at time $t$ after $n_{\bm{a}}(t)$ pulls of the joint arm $\bm{a}$. Note the empirical mean is indexed by the player since each player gets their own copy of the IID reward. However, as the rewards for each joint arm have the same mean across all players, the expected regret $R_T$ is the same for each player. On the other hand, $n_{\bm{a}}$ is not indexed by the player $i$ since each player is contributing to the same joint action $\bm{a}$ at every round. 

Note that although each player will likely get different rewards, the rewards are iid so that in expectation the regret for all the players are the same. Note that fundamental results for single-player MAB problems \cite{lai1985asymptotically} suggest a $O(\log T)$-regret lower bound for the multi-player MAB problem as well. If we can design a multi-player decentralized learning algorithm with such a regret order, then it would imply that such a lower bound is tight for this setting as well. 
\section{Main Results}\label{sec:main_results}
To maximize the cumulative rewards, the players will define the UCB (Upper Confidence Bound) index for each joint action (not just their own action set), and try to pull arms with high UCB indices as often as possible. This index for player $i$ is given by
\begin{equation}\label{index:mUCB}
	\eta^i_{\bm{a}}(t) =  \begin{cases}
		\infty, & \text{if } n_{\bm{a}}(t) = 0,\\
\hat{\mu}_{\bm{a}}^i(n_{\bm{a}}(t)) + \sqrt{\frac{2\log(1/\delta)}{n_{\bm{a}}(t)}}, & \text{otherwise}.
	\end{cases}
\end{equation}
Note that $\eta_{\bm{a}}^i$ is indexed by the player because each player observes a different empirical reward mean $\hat{\mu}_{\bm{a}}^i(n_{\bm{a}}(t))$. 

Let $\epsilon_{\bm{a}}(\cdots)$ be the constant added to the empirical mean for arm $\bm{a}$ in calculating the UCB index. This constant is the same for every player which is why is not indexed by $i$. Since the intervals used in our algorithm \ref{algo} will be of length $2\epsilon_{\bm{a}}(\cdots)$, we add a hyperparameter $\gamma$ to tune this interval length. More explicitly, it is defined as  
	\begin{equation}
	    \epsilon_{\bm{a}}( n_{\bm{a}}, \delta, \gamma):= \gamma\sqrt{\frac{\log(1/\delta)}{n_{\bm{a}}}}
	\end{equation}

where $t$ is the round number, $n_{\bm{a}}$ is the number of times arm $\bm{a}$ has been fulled, $\delta$ a constant that is set to $\frac{1}{T^2}$, and $\gamma$ scales the length of our interval. We know from Hoeffding's bound that for subgaussian variables, the true mean $\mu_{\bm{a}}$ is within the interval 
\begin{equation}\label{eq:I}
I_{\bm{a}}^i = (\hat{\mu}_{\bm{a}}^i- \epsilon_{\bm{a}}, \hat{\mu}_{\bm{a}}^i + \epsilon_{\bm{a}})
\end{equation} with high probability, where we have omitted the arguments as they are clear from the context. Using this, at each round we can create a set of arms that are likely to be optimal and make our selection from this set. Note that each player $i$ has a different interval for joint action $\bm{a}$ due to the empirical means $\hat{\mu}_{\bm{a}}^i$ being different. However, for each joint action $\bm{a}$, every player has intervals of the same length for this action. Thus $\epsilon_{\bm{a}}$ is not indexed by the player $i$ since it is the same for each player. 

We now propose \texttt{mUCB-Intervals} in Algorithm \ref{algo} used to deal with reward asymmetry.
 In this algorithm, all the players will maintain a \emph{desired set} which contains the joint arms that are candidates for the optimal arm. Initially, all $K^m$ joint actions are in this desired set. To ensure coordination, all the players will agree before the learning process on an order for the set of joint actions in this desired set. This is similar to how \cite{chang2021online} dealt with asymmetry in actions. Furthermore, by observing the actions of the other players, they will be able to maintain the same desired sets at each round. At a particular round, a joint action is called \emph{considered} if it's the arm in the desired set that was supposed to be pulled that round in accordance with the order that was agreed upon by all the players. We will use $\bm{c}$ to denote a joint action in the desired set. If there are $\ell$ joint actions $\bm{c}_1,\ldots,\bm{c}_\ell$, then the ordering of the desired set can be viewed using the following flow chart.
\begin{equation}\label{eq:flow}
  \begin{tikzcd}
\bm{c}_1 \to \bm{c}_2\to \cdots  \to \bm{c}_{\ell} \arrow[bend left = 10,
start anchor={[xshift= 10 ex, yshift = -1.5ex]},
end anchor={[xshift= -9ex, yshift = -1.5ex]}]
\end{tikzcd}
\end{equation}

\begin{algorithm}
\caption{\texttt{mUCB-Intervals}}
Each player $P_i$ has all the joint arms in their \emph{desired sets}. All the players will agree on the ordering of the joint arms.\

\For{$t = 1,\ldots,K^M$}{
Each player $i$ will pull each joint action $\bm{a}$ once in the order they have decided in advance, and update $I_{\bm{a}}^i$. \
}
\For{$t = K^M+1,\ldots, T$}
{
    Each player $i$ identifies the next arm considered $\bm{c}_{t}$ in the desired set based on the arm pulled in the previous round (see flowchart in \eqref{eq:flow}).\
    
    \eIf{exists player $i$ and joint action $\bm{a'}$ such that $I_{\bm{a'}}^i$ is above and disjoint from $I_{\bm{c}_t}^i$}{
    Player $i$ will not pull $\bm{c}_t[i]$ to inform the other players that he will remove $\bm{c}_{t}$ from his desired set. \
    }
    {
    Each player $i$ pull $\bm{c}_t[i]$.\ 
    }
    Each player $i$ observes the actions from other players to determine the joint action taken $\bm{a}_t$ at that step. They observe their own i.i.d. reward, and update their $I_{\bm{a}_t}^i$. \
    
    \If{$\bm{a}_t \neq \bm{c}_t$}{
    All players eliminate $\bm{c}_{t}$ from their desired set whilst maintaining the same ordering of the remaining arms (see flowchart in \eqref{eq:flow_eliminate}) \
    }
}
\label{algo}
\end{algorithm}

A joint arm is eliminated from this desired set, if at a particular round $t$, a player $i$ observes that the considered joint arm (call it $\bm{a}_{k_t}$ in the flow diagram above) has a UCB interval that is below and disjoint from another arm. Then that player $i$ will effectively 'communicate' this elimination to the other players by not pulling $\bm{c}_{t}[i]$. Other players will observe that $\bm{c}_{t}[i]$ was not pulled by player $i$ and eliminate it from their desired sets as well whilst maintaining the same ordering of the remaining arms. In other words, pulling a joint arm that isn't considered is the same as removing the considered arm from the desired set. The new flow diagram after the elimination is,
\begin{equation}\label{eq:flow_eliminate}
  \begin{tikzcd}
\bm{c}_1 \to \bm{c}_2\to  \cdots\bm{c}_{t-1}\to \bm{c}_{t+1}\to \cdots  \to \bm{c}_{\ell} \arrow[bend left = 10,
start anchor={[xshift= 20 ex, yshift = -1.5ex]},
end anchor={[xshift= -19ex, yshift = -1.5ex]}]
\end{tikzcd}
\end{equation}
and the arm that is considered in the next round $t+1$ is the next arm in the ordering prior to the elimination. In the flowchart \eqref{eq:flow_eliminate} it is $\bm{a}_{k_t+1}$. On the other hand, if the arm that was considered is the same as the one that was pulled, then the flowchart in \eqref{eq:flow} remains unchanged.

\subsection{Example}\label{example}
Due to the novelty of this algorithm and setting, we will present an example of how this algorithm runs. Consider a two-player setting where each player has two arms. We will represent each joint arm as an entry in a matrix and what goes inside the matrix is the corresponding UCB interval defined in equation \eqref{eq:I}. The subscript of the matrix will correspond to the player. The rows are numbered by player 1's actions while the columns are numbered by player 2's actions. Therefore, our initial matrix for player $1$ and player $2$ respectively  is  
$$
    \kbordermatrix{ & 1 & 2  \\
     1 &    (-\infty, \infty) & (-\infty, \infty) \\ 
     2& (-\infty, \infty) & (-\infty, \infty)
  }_1, \kbordermatrix{ & 1 & 2  \\
     1 &    (-\infty, \infty) & (-\infty, \infty) \\ 
     2& (-\infty, \infty) & (-\infty, \infty)
  }_2
  $$
   Suppose the order of the joint actions is
   
\begin{equation}
  \begin{tikzcd}
(1, 1)\to (1, 2)\to (2, 1) \to (2,2)\arrow[bend left = 10,
start anchor={[xshift= 15 ex, yshift = -1.5ex]},
end anchor={[xshift= -14ex, yshift = -1.5ex]}]
\end{tikzcd}
\end{equation}
Initially, all the UCB intervals are infinite and each arm belongs in the desired set so each arm gets pulled at least once in order. Suppose the matrices for both players after pulling these four joint actions from $t = 1$ to $t = 4$ are now 
\begin{equation}
    \kbordermatrix{ & 1 & 2  \\
     1 &    (.2, .5) & (.3, .6) \\ 
     2& (.55, .9) & (.65, .8)
  }_1, \kbordermatrix{ & 1 & 2  \\
     1 &    (.5, .7) & (.4, .7) \\ 
     2& (.6, .9) & (.65, .8)
  }_2
\end{equation}
In the following round $t = 5$, the next arm to pull in order is $(1, 1)$ because once you reach the end of the desired arm set you start over from the beginning. Thus the action that is \emph{considered} at this round is $(1, 1)$. However, note that the interval for arm $(1, 1)$ for player $1$ is now disjoint from the interval for arm $(2, 1)$. Therefore, player $1$ knows with a high probability that this arm is not optimal and is thus ready to eliminate it. In order to communicate that to player $2$, he will pull arm $2$ instead of arm $1$. Player 2 had no intention of eliminating arm $(1, 1)$ so he will pull arm $1$. Thus the joint action that was taken at $t=5$ is $(2, 1)$. Player $2$ will observe that player $1$ pulled a different arm than what was agreed upon, and understands that $(1, 1)$ needs to be eliminated. Thus the desired set for both players now only contains, 
   \begin{equation}
  \begin{tikzcd}
(1, 2)\to (2, 1) \to (2,2)\arrow[bend left = 10,
start anchor={[xshift= 10 ex, yshift = -1.5ex]},
end anchor={[xshift= -9ex, yshift = -1.5ex]}]
\end{tikzcd}
\end{equation}

The next \emph{considered} arm for $t = 6$ is $(1, 2)$ because that was the arm right after $(1, 1)$ prior to the elimination, and this process repeats itself until eventually there is only $1$ arm left which will be the optimal arm with high probability. 

\subsection{Discussion}
We discuss why this algorithm performs so much better than \texttt{mDSEE} from \cite{chang2021online} which was used in asymmetry in both rewards and actions. \texttt{mDSEE} explores all arms equally at exponentially increasing intervals, which incurs a lot of regret during the exploration process. In comparison, \texttt{mUCB-Intervals} only focuses on arms that can potentially be the optimal one, by eliminating arms that are clearly suboptimal immediately. 

Furthermore, \texttt{mDSEE} requires a unique global optimum. This is because if there are two optimal actions, then at the committing phase, players will commit to different joint actions with constant probability. However, since the desired sets of all the players are the same, in this setting a unique global optimal action is not required. 

This algorithm is similar to best arm identification algorithms such as those in \cite{audibert2010best}, where they also study UCB intervals to determine which arm is optimal. However, in our case, the predefined ordering of the joint arms and maintenance of the same desired set for each player are what makes coordination possible.  

\subsection{Regret Bounds}
We have the following regret bounds for \texttt{mUCB-Intervals}.
\begin{theorem}\label{thm:mUCB-gap}
    For any choices of $\gamma > 0$, the gap dependent regret of algorithm \texttt{mUCB-Intervals} is $R_T = O\left(\left(\sum_{\bm{a} \in \mathcal{A}}\frac{1}{\Delta_{\bm{a}}}\right)\log T + MK^2\right)$.
\end{theorem}
\begin{theorem}\label{thm:mUCB-nogap}
    For any choices of $\gamma > 0$, the gap-independent regret bound of Algorithm \ref{algo} is $R_T = O(\sqrt{KT\log(T)})$.
    \end{theorem}

\section{Proofs}
In this section, we present the proofs to the regret bounds given in section \ref{sec:main_results}. The following property on subgaussian variables will be important. 
\begin{lemma}
	Assume that $X_i - \mu$ are independent, $\sigma$-subgaussian random variables. Then, for any $\epsilon \geq 0$,
	\begin{equation}
        \begin{split}
		P\left(\hat{\mu} \geq \mu + \epsilon\right) \leq \exp\left(-\frac{T\epsilon^2}{2\sigma^2}\right)\quad  \\
          \text{and} \quad P\left(\hat{\mu} \leq \mu - \epsilon\right) \leq \exp\left(-\frac{T\epsilon^2}{2\sigma^2}\right)
        \end{split}
	\end{equation}
	where $\hat{\mu} = \frac{1}{T}\sum_{t=1}^T X_t$. 
	\label{corollary5.5}	
\end{lemma}	
\begin{proof}
	The reader is encouraged to look at Corollary 5.5 of \cite{lattimore2020bandit}. It relies on the observation that $\hat{\mu} - \mu$ is $\sigma/\sqrt{T}$-subgaussian. 
\end{proof}

The following regret decomposition is also going to be useful in our proofs. 

\begin{lemma}\label{lem:regret_decomp}
	With regret defined in \eqref{eq:regret}, we have the following regret decomposition for each player $i$:
	\begin{equation}\label{eq:regret_decomp}
		R_T^i = \sum_{\bm{a}} \Delta_{\bm{a}}\mathbb{E}\left[n_{\bm{a}}(T)\right],
	\end{equation}
	where $n_{\bm{a}}(T)$ is the number of times the arm $\bm{a}$ has been pulled up to round $T$.
 
 \begin{proof}
     The reader is encouraged to look at Lemma 4.5 of \cite{lattimore2020bandit}. It follows from the fact that each round you pull arm $\bm{a}$, you incur (in expectation) $\Delta_{\bm{a}}$ regret. 
 \end{proof}
\end{lemma} 	

We are now ready to present the proof of Theorem \ref{thm:mUCB-gap}
\begin{proof}[Proof of Theorem \ref{thm:mUCB-gap}] 
We will write $\hat{\mu}_{\bm{a}}^i(n_{\bm{a}}(t))$ as $\hat{\mu}_{\bm{a}}^i$ when it is clear from the context what the argument should be.

We suppose that the first arm is the optimal one for each player, that is arm $\bm{1} = \left(1,\ldots,1\right)$ has the highest average reward. We define the following "good" event $G$ where all arms have their true means in the intervals at all times. Explicitly, this is written as 
    
    \begin{equation}
        G = \bigcap_{i=1}^M \bigcap_{\bm{a} \in \mathcal{A}} \{|\hat{\mu}_{\bm{a}}^i - \mu_{\bm{a}}|<\epsilon_{\bm{1}}(t ,\delta, \gamma)|\forall t \in [1, n]\}
    \end{equation}


 
 We define another good event for each arm based on the observed means of arm $\bm{a}$ and arm $1$. 
	\begin{equation}\label{eq_good}
G_{\bm{a}}^{\text{cross}}= \bigcap_{i=1}^M \{ \hat{\mu}_{\bm{a}}^m(u_{\bm{a}}, \delta, \gamma) + \epsilon_{\bm{a}}(u_{\bm{a}}, \delta, \gamma) < \mu_{\bm{1}} - \epsilon_{\bm{1}}(u_{\bm{a}}-1, \delta, \gamma)\}
	\end{equation}

 This set is the event that after $u_{\bm{a}}$ pulls, the UCB indices of all the arms $\bm{a}$ is strictly smaller than lower bound of UCB index of arm $\bm{1}$.
 The next lemma shows that when $G \cap  G_{\bm{a}}^{\text{cross}}\cap  G_{\bm{a}}^{\text{pull}}$ occurs, the number of pull is at most $u_{\bm{a}}$. 
 
 \begin{lemma}
     Under the event $G \cap  G_{\bm{a}}^{\text{cross}}$, the number of pulls of arm $\bm{a}$ is at most $u_{\bm{a}}$.
     
     \begin{proof}
     Under the event $G$ the optimal arm $\bm{1}$ is always in the desired set. As the arms are pulled one at at time in a predefined order in the desired set, if an arm $a$ is at the desired set at time $t$ then $n_{\bm{1}}(t) \geq n_{\bm{a}}(t) -1$. It follows that at the time when $n_{\bm{a}}(t) = u_{\bm{a}}$, we have $n_{\bm{1}}(t) \geq u_{\bm{a}} - 1$. Furthermore, $ \epsilon_{\bm{1}}(u_{\bm{a}}-1, \delta, \gamma)$ is decreasing as a function of the number of pulls so $\epsilon_{\bm{1}}(u_{\bm{a}}-1, \delta, \gamma) \geq \epsilon_{\bm{1}}(n_{\bm{1}}(t), \delta, \gamma)$
     
         Suppose for the sake of contradiction that $n_{\bm{a}}(n) > u_{\bm{a}}$. Then there must exist a round $t$ such that $n_{\bm{a}}(t - 1) = u_{\bm{a}}$ and the action taken at step $t$ was $\bm{a}$,
         
         \begin{align*}
             \eta_{\bm{a}}^i(t - 1) &= \hat{\mu}^m_{\bm{a}}(\cdot, u_{\bm{a}}) + \gamma\sqrt{\frac{\log(1/\delta)}{u_{\bm{a}}}}\\
             &< \mu_{\bm{1}} - \epsilon_{\bm{1}}(u_{\bm{a}}-1, \delta, \gamma)\text{ since } G_{\bm{a}}^{\text{cross}} \text{ occurs}\\
             &\leq \mu_{\bm{1}} - \epsilon_{\bm{1}}(n_{\bm{1}}(t), \delta, \gamma)\text{ since }n_{\bm{1}}(t) \geq u_{\bm{a}} - 1
         \end{align*}
     
        Thus, for each player, it follows that the interval corresponding to arm $\bm{a}$ is disjoint and below the interval corresponding to arm $1$. Thus, arm $\bm{a}$ should have been eliminated from desired set, by round $t$. This contradiction completes the proof. 
     \end{proof}
 \end{lemma}

    Using the regret decomposition lemma, we can aim to bound $\mathbb{E}[n_{\bm{a}}(t)]$. We can decompose this quantity as follows, 

 \begin{align*}   
 \mathbb{E}[n_{\bm{a}}(T)] &=  \mathbb{E}[n_{\bm{a}}(T)I(G \cap  G_{\bm{a}}^{\text{cross}}] + \mathbb{E}[n_{\bm{a}}(T)I((G \cap  G_{\bm{a}}^{\text{cross}})^c)] \\    
 &= \mathbb{E}[n_{\bm{a}}(T)\mathbb{I}(G \cap  G_{\bm{a}}^{\text{cross}})] \\
 &+ \mathbb{E}[n_{\bm{a}}(T)(\mathbb{I}(G^c) + I(G \cap (G_{\bm{a}}^{\text{cross}})^c))]\\  
 &\leq u_{\bm{a}} + (P(G^c)+ P(G \cap(G_{\bm{a}}^{\text{cross}})^c))T
 \end{align*}
We start with $P(G^c)$: 
 \begin{align*}
	  P(G^c)&= P\left(\bigcup_{i=1}^M \bigcup_{\bm{a}} \left(\{\exists t \in [1, n]: |\hat{\mu}_{\bm{a}}^i - \mu_{\bm{a}}|\geq \epsilon_{\bm{a}}(t ,\delta, \gamma)\} \right)\right)\\
	  &= \sum_{i=1}^M \sum_{\bm{a}} P\{\exists t \in [1, n]: |\hat{\mu}_{\bm{a}}^i - \mu_{\bm{a}}|\geq \epsilon_{\bm{a}}(t ,\delta, \gamma)\} \\
	  &= M\sum_{\bm{a}}\left(\sum_{t=1}^T P\{|\hat{\mu}_{\bm{a}}^i - \mu_{\bm{a}}|\geq \epsilon_{\bm{a}}(t ,\delta, \gamma)\} \right) \\
	  &\leq M\sum_{\bm{a}}\left(\sum_{t=1}^T 2\exp{-\frac{t\epsilon_{\bm{a}}(t, n_{\bm{a}}, \delta, \gamma)^2}{2}} \right)\\
	  &\leq M\sum_{\bm{a}}\left(\sum_{t=1}^T 2\delta^{\frac{t\gamma}{2n_{\bm{a}}}}\right)\\
	  &\leq MKT\delta^\gamma
	\end{align*}
Finally, we turn to $P(G\cap (G_{\bm{a}}^{\text{cross}})^c)$. 
	\begin{align*}
	&P(G \cap (G_{\bm{a}}^{\text{cross}})^c)\\
 &= P\left(\bigcup_{i=1}^M  \{ \hat{\mu}_{\bm{a}}^m(u_{\bm{a}}, \delta, \gamma) + \epsilon_{\bm{a}}(u_{\bm{a}}, \delta, \gamma) \geq \mu_{\bm{1}} - \epsilon_{\bm{1}}(u_{\bm{a}} - 1, \delta, \gamma)\} \right)\\
	  &= M P\{ \hat{\mu}_{\bm{a}}^m(u_{\bm{a}}, \delta, \gamma) + \epsilon_{\bm{a}}(u_{\bm{a}}, \delta, \gamma) \geq \mu_{\bm{1}} - \epsilon_{\bm{1}}(u_{\bm{a}} - 1, \delta, \gamma)\}\\
	  &= M P\{ \hat{\mu}_{\bm{a}}^m(u_{\bm{a}}, \delta, \gamma)  - \mu_{\bm{a}} \geq \Delta_{\bm{a}} - \epsilon_{\bm{a}}(u_{\bm{a}}, \delta, \gamma) - \epsilon_{\bm{1}}(u_{\bm{a}} - 1, \delta, \gamma)\}\\
	  &\leq M\exp\left(-\frac{u_{\bm{a}}(\Delta_{\bm{a}}- \epsilon_{\bm{a}}(u_{\bm{a}}, \delta, \gamma) - \epsilon_{\bm{1}}(u_{\bm{a}} - 1, \delta, \gamma))^2}{2}\right)
	\end{align*}

Let us pick $u_{\bm{a}} = \left\lceil\frac{9\gamma^2\log(1/\delta)}{(1-c)^2\Delta_{\bm{a}}^2}\right\rceil$ for some $c \in (0, 1)$, and this will satisfy,
\begin{align*}
    \Delta_{\bm{a}}- \epsilon_{\bm{a}}(u_{\bm{a}}, \delta, \gamma) - \epsilon_{\bm{1}}(u_{\bm{a}} - 1, \delta, \gamma) &\geq  \Delta_{\bm{a}}- 3\epsilon_{\bm{a}}(u_{\bm{a}}, \delta, \gamma)\\
    & \geq c\Delta_{\bm{a}}
\end{align*}

Combining everything together, we get

\begin{equation}\begin{split}    & P((G_{\bm{a}}^{\text{cross}}\cap G)^c) \\ &\leq  MKT\delta^\gamma +  M\exp\left(\frac{-u_{\bm{a}} c^2\Delta_{\bm{a}}^2}{2}\right)\\   &\leq MKT\delta^\gamma + M\delta^{\frac{\gamma^2\log(1/\delta)c^2}{(1-c)^2\Delta_{\bm{a}}^2}}    \end{split}\end{equation}

Choosing $\delta = \frac{1}{T^{2/\gamma}}$, we obtain the following regret bound
\begin{align*}
   \mathbb{E}[n_{\bm{a}}(T)] &\leq \left\lceil\frac{18\gamma^2\log(1/\delta)}{(1-c)^2\Delta_{\bm{a}}^2}\right\rceil\\
   &+ \left(  MKT\delta^\gamma +  M\delta^{\frac{9\gamma^2\log(1/\delta)c^2}{(1-c)^2\Delta_{\bm{a}}^2}} \right)T\\
   &= \left\lceil\frac{18\gamma\log(T)}{(1-c)^2\Delta_{\bm{a}}^2}\right\rceil + MK + MT^{-\frac{9\gamma c^2}{(1-c)^2\Delta_{\bm{a}}^2}+1}
\end{align*}

Our goal is now to select $c$ so that the exponent of $T$ in the expression above is negative. In order for $-\frac{9\gamma c^2}{(1-c)^2\Delta_{\bm{a}}^2}+1\leq 0$ to hold, it is sufficient that $c \geq \frac{\Delta_{\bm{a}}}{\Delta_{\bm{a}}+3\sqrt{\gamma}}$. Furthermore, it's clear that $c \in (0, 1)$ as originally stated. Thus, with this value of $c$, our bound above becomes 
\begin{align}
\mathbb{E}[n_{\bm{a}}(T)] &\leq\left\lceil\frac{18\gamma\log(T)}{(1-c)^2\Delta_{\bm{a}}^2}\right\rceil + M(K+1) + 1 \\ 
&\leq \left\lceil\frac{2(\Delta_{\bm{a}} + 3\sqrt{\gamma})^2\log(T)}{\Delta_{\bm{a}}^2}\right\rceil + M(K+1) + 1 \label{eq:n} 
\end{align}
Plugging this into the regret decomposition shows $R_T = O(\log(T))$, completing the proof of theorem \ref{thm:mUCB-gap}
\end{proof}  

While our analysis seems to suggest that the smaller the $\gamma > 0$, the better the performance, however, note that when $\gamma$ is too small, there is a higher probability that a good arm is eliminated. In the experiments, we show that when $\gamma$ is sufficiently small, it can outperform even the mUCB algorithm from \cite{chang2021online} which is a coordinated version of the UCB algorithm in single player setting. 

\begin{proof}[Proof of Theorem \ref{thm:mUCB-nogap}]
    	We first take the regret decomposition given by equation \eqref{eq:regret_decomp} and partition the tuples $\bm{a}$ to those whose mean are at most $\epsilon$ away from the optimal  and those whose means are more than $\epsilon$. This gives us the following inequality: 
	\begin{align}
		R_T &= \sum_{\bm{a}} \Delta_{\bm{a}}\mathbb{E}\left[n_{\bm{a}}(T)\right]\\
  &= \sum_{\Delta_{\bm{a}} > \epsilon} \Delta_{\bm{a}}\mathbb{E}\left[n_{\bm{a}}(T)\right] + \sum_{\Delta_{\bm{a}} \leq \epsilon} \Delta_{\bm{a}}\mathbb{E}\left[n_{\bm{a}}(T)\right]\\
  &\leq \sum_{\Delta_{\bm{a}} > \epsilon} \Delta_{\bm{a}}\mathbb{E}\left[n_{\bm{a}}(T)\right] + \epsilon T.
	\end{align}
 Equation \eqref{eq:n}
 yields:
	\begin{align}
		R_T &\leq \sum_{\Delta_{\bm{a}}> \epsilon}O\left( \frac{\log(T)}{\Delta_{\bm{a}}}\right) + \epsilon T \\
  & \leq \sum_{\Delta_{\bm{a}}> \epsilon}O\left( \frac{\log(T)}{\epsilon}\right) + \epsilon T 
	\end{align}
	Since the last inequality holds for all $\epsilon$, we can pick $\epsilon = \sqrt{\frac{\log(T)}{T}}$ to obtain the result in Theorem \ref{thm:mUCB-nogap}.

\end{proof}

\begin{figure*}
\includegraphics[width = \textwidth]{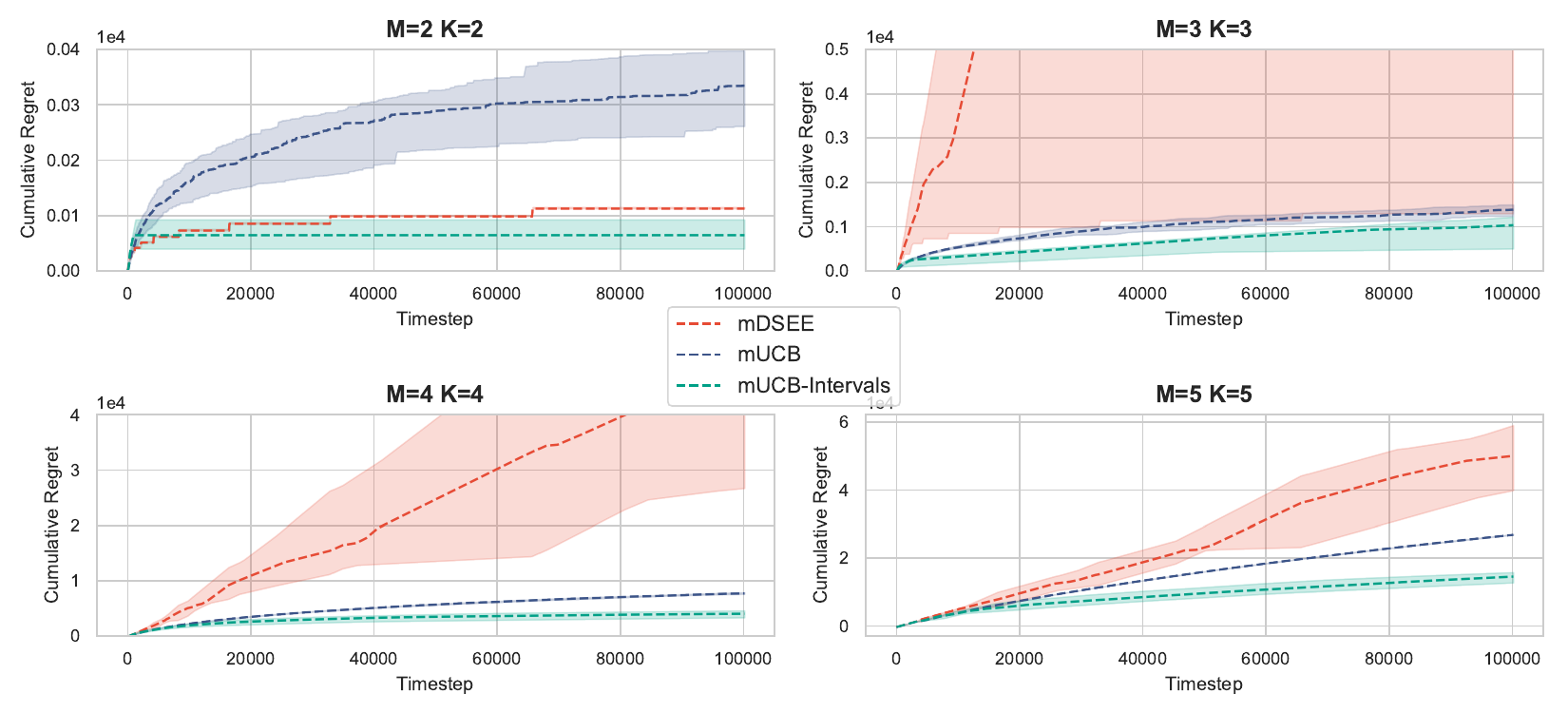}
\caption{Plots comparing the regret of mUCB \cite{chang2021online}, \texttt{mUCB-Intervals}, and mDSEE \cite{chang2021online} under the same reward environment but under asymmetry in actions, asymmetry in rewards, and asymmetry in both respectively for horizon $T = 10^5$. The shaded regions are 95\% confidence intervals. It's clear to see that the algorithm proposed in this paper \texttt{mUCB-Intervals} outperforms the SOTA algorithm for asymmetry in both (mDSEE) by a large margin. The superiority of \texttt{mUCB-Intervals} is more pronounced as the joint action space increases in size.}
\label{fig:experiments}
\end{figure*}

\section{Numerical Simulations}\label{sec:numerical}

In this section, we run simulations to verify the empirical performance of our algorithm. In Figure \ref{fig:experiments}, we plot the regret versus time of \texttt{mUCB-Intervals} analyzed in this paper in comparison with mDSEE and mUCB from \cite{chang2021online}. It should be emphasized these algorithms assume different types of asymmetry. For mDSEE, both action\footnote{Action asymmetry refers to the scenarios where agents do not observe the action taken by other agents, and thus the joint-action taken remains unknown. In general, action asymmetry is usually easier to solve than reward asymmetry because a proper coordination scheme would reduce action asymmetry to a single-player problem.} and reward asymmetry are assumed. mUCB assumes only action asymmetry, and \texttt{mUCB-Intervals} assume only reward asymmetry.

We perform these simulations with Gaussian rewards sampled from distributions with means sampled uniformly from 0 to 1 and standard deviations uniformly from 0 to 0.5. For each environment, we run the simulations for a total of $T = 100,000$ rounds and repeat the simulations for $10$ times to plot both the median and the $95\%$ confidence interval of regret. We set the hyperparameter $\gamma = 0.5$ for \texttt{mUCB-Intervals}. This choice is arbitrary and the performance of \texttt{mUCB-Intervals} is not significantly affected by the choice of $\gamma$ for non-extreme values.

In all simulations, we observe log-like behaviors for the \texttt{mUCB-Intervals} algorithm. It is clear from the plots that \texttt{mUCB-Intervals} exhibits an absolute competitive advantage over other algorithms as their confidence intervals on regret become disjoint in longer horizons. For robustness, \texttt{mUCB-Intervals} also has a significant advantage over mDSEE, as demonstrated by the large red-shaded regions.

These discrepancies in performance can be explained by the differences in the type of asymmetries the algorithms consider. Since \texttt{mDSEE} also deals with asymmetry in actions, when too many actions have the same empirical mean, it becomes easy for the players to mis-coordinate. On the other hand, \texttt{mUCB-Intervals} utilizes observation of actions from the other players to coordinate the actions, giving it a superior performance. On the other hand, we see that \texttt{mUCB-Intervals} outperforms the coordinated UCB algorithm mUCB from \cite{chang2021online} as well. This is because for small values of $\gamma$, there is less exploration so that the \texttt{mUCB-Intervals} is able to eliminate suboptimal arms faster and commit to better arms more often. However, one must be careful to not choose a value of $\gamma$ that is too small, as such a value would discourage sufficient exploration. This will cause god arms to be eliminated too quickly. 

In more detail, in the first few rounds, the UCB intervals for \texttt{mUCB-Intervals} are large, and thus all of them are included in the set for all players. As the rounds progress, the intervals shrink as they tend towards the true means of the arms. Some sub-optimal arms then get eliminated from the set, giving a higher probability that a good arm is pulled. Eventually, with high probability, all the players will only have 1 arm in their set - the optimal one. From that point forward, the regret curve is horizontal because no additional suboptimal arms are pulled. It is clear from this plot that we have sub-log regret.


\section{Conclusions and Future work}
In this paper we considered a cooperative multiplayer bandit learning problem where the players are only allowed to agree on a strategy beforehand, but cannot communicate during the learning process. The actions of all the players are commonly observed. However, each player receives a noisy version of the reward which cannot be shared with other players. We provide an algorithm \texttt{mUCB-Intervals} based on upper and lower confidence bounds scaled by $\gamma$ that the players can use to select their optimal actions despite the asymmetry in the reward information. For any choice of $\gamma$ we show that this algorithm can achieve logarithmic (gap-dependent) regret as well as $O(\sqrt{T}\log T)$ gap-independent regret giving us asymptotically optimal regret for our problem. We ran numerical simulations on multiplayer bandit problem and compared it with \texttt{mDSEE} from \cite{chang2021online}, and saw that some choices of $\gamma$ perform better while others don't. For future work, we can better understand what choices of $\gamma$ lead to a better performance. We can remove the asymmetry in actions (i.e. consider a more general setting where the players cannot observe the other player's action either) and try to derive an algorithm that gives asymptotically optimal regret. We can also consider a bandit MDP setting, where each joint action now changes the environment for everyone. In this setting we can still consider, asymmetry in rewards, asymmetry in actions, or both. 

\newpage

\bibliographystyle{abbrvnat}
\bibliography{references}

\end{document}